\documentclass[letterpaper]{article}

\usepackage{times}  
\usepackage{helvet} 
\usepackage{courier}  
\usepackage[hyphens]{url}  
\usepackage{graphicx} 
\usepackage[ bottom=1.25in, top=1.25in, left=1in, right=1in]{geometry}
\usepackage{float}

\usepackage{caption} 
\usepackage{authblk}
\usepackage{hyperref}

\usepackage{xcolor}
\usepackage{soul}

\usepackage{enumitem}
\usepackage{amsthm, amssymb, amsmath}
\usepackage{cleveref}
\usepackage{caption} 
\usepackage{subcaption} 
\usepackage{todo}
\usepackage{graphicx}
\usepackage{algorithm}
\usepackage[noend]{algpseudocode}
\usepackage{breakcites}

\newtheorem{theorem}{Theorem}
\newtheorem*{theorem*}{Theorem}

\newtheorem*{claim*}{Claim}

\newtheorem*{proposition*}{Proposition}
\newtheorem{lemma}[theorem]{Lemma}
\newtheorem*{lemma*}{Lemma}

\newtheorem*{conjecture*}{Conjecture}

\newtheorem*{fact*}{Fact}

\newtheorem*{hypothesis*}{Hypothesis}

\theoremstyle{definition}
\newtheorem{definition}[theorem]{Definition}

\DeclareMathOperator{\Tr}{Tr}
\DeclareMathOperator{\cost}{cost}
\DeclareMathOperator{\fcost}{fair-cost}

\usepackage{makecell, cellspace, multirow}

\title{Socially Fair Center-based and Linear Subspace Clustering}
\date{}

\author[1]{Sruthi Gorantla}
\author[2]{Kishen N Gowda}
\author[3]{Amit Deshpande\(^*\)}
\author[1]{Anand Louis\thanks{Equal contribution}}
\affil[1]{\small{Indian Institute of Science, Bengaluru, India. \texttt{\{gorantlas, anandl\}@iisc.ac.in}}}
\affil[2]{\small{University of Maryland, College Park, US. \texttt{kishen19@cs.umd.edu}}}
\affil[3]{\small{Microsoft Research, Bengaluru, India. \texttt{amitdesh@microsoft.com}}}

\renewcommand{\leq}{\leqslant}
\renewcommand{\le}{\leqslant}
\renewcommand{\geq}{\geqslant}
\renewcommand{\ge}{\geqslant}

\newcommand{\paren}[1]{\left(#1 \right )}

\newcommand{\set}[1]{\left\{#1\right\}}

\newcommand{\Z}{{\mathbb Z}}

\newcommand{\R}{\mathbb R}
\newcommand{\Q}{\mathbb Q}

\newcommand{\cO}{\mathcal{O}}
\newcommand{\cP}{\mathcal{P}}
\newcommand{\cC}{\mathcal{C}}

\newcommand{\opt}{{\sf OPT}}
\newcommand{\oracle}{{\sc Fair-Centers}}

\newcommand{\E}{\mathbb{E}}

\allowdisplaybreaks
\usepackage[parfill]{parskip}

\begin{document}

\maketitle

\begin{abstract}
    Center-based clustering (e.g., $k$-means, $k$-medians) and clustering using linear subspaces are two most popular techniques to partition real-world data into smaller clusters.
    However, when the data consists of sensitive demographic groups, significantly different clustering cost per point for different sensitive groups can lead to fairness-related harms (e.g., different quality-of-service).
    The goal of socially fair clustering is to minimize the maximum cost of clustering per point over all groups.
    In this work, we propose a unified framework to solve socially fair center-based clustering and linear subspace clustering, and give practical, efficient approximation algorithms for these problems.
    We do extensive experiments to show that on multiple benchmark datasets our algorithms either closely match or outperform state-of-the-art baselines. 
\end{abstract}
\section{Introduction}
\label{sec:problem}
Given a set of $n$ data points in a $d$-dimensional space described by their feature representation, the goal of clustering is to partition these data points into $k$ disjoint parts or clusters so that the points in the same cluster are close to each other or close to a well-defined structure. 
Clustering has a wide range of applications in computer science \cite{Karl1994}, computational biology \cite{HN2000}, market segmentation \cite{Haben2016AnalysisAC}, social science \cite{MSST2007}, field robotics \cite{BSLUL2011}, and many more. Clustering point clouds using subspaces has applications in computer vision \cite{Ho2003clustering,Vidal2010tutorial} and face recognition \cite{Ho2005acquiring,Batur2004segmented}. Subspaces fitting high-dimensional word and sentence embeddings have been shown to reveal interesting linguistic phenomena \cite{Mu2017geometry,Mu2017representing}.
Previous work has shown that high-dimensional representations of text data such as commonly used word embeddings contain societal biases \cite{bolukbasi2016man}, and data mining and algorithmic decision-making on big data can have potentially adverse social and economic impact on individuals and sensitive demographic groups (e.g., race, gender) \cite{Barocas2016big}. As a result, various fairness-constrained clustering objectives have been proposed and studied in recent literature as follows.

It is a challenging problem in itself to efficiently determine the right number of clusters purely from the data \cite{PellegM2000xmeans,HamerlyE2003learning}.  
As a result, in many clustering objectives, the desired number of clusters $k$ is given as a part of the input.
In center-based clustering objectives, the objective is to find $k$ points or centers that minimize the average distance of any data point to its nearest center. We use the Euclidean distance in all our results. Using $z$-th power of the Euclidean distance metric gives rise to well-known clustering objectives such as $k$-median ($z=1$), $k$-means ($z=2$), $k$-center ($z=\infty$).
More formally, consider the following center-based clustering problem,
\begin{definition}[\textbf{$(k,z)$ clustering}]
\label{def:kz_clustering}
	Given a set $X$ of $n$ points in $\R^d$, a weight function $w :  X \rightarrow \R_{>0}$, a positive integer $k$ and a real number $z \in [1, \infty)$, the objective of $(k,z)$ clustering is to find $k$ centers $C = \paren{c_1, \ldots, c_k}$, with each $c_i \in \R^{d}$, that minimize the weighted cost of clustering defined as
	\[
	\cost(C,X,w) := \left(\sum_{x \in X} w(x)\cdot d(C,x)^z\right)^{1/z},\quad
	\text{where}~~d(C,x):=\min_{i\in [k]} d(c_i,x) = \min_{i \in [k]}\|c_i - x\|_2.
	\]
\end{definition}
We consider uniform weights for all the points in $X$, that is, $w(x) = \frac{1}{|X|}$, unless stated otherwise.

In clustering using $k$ subspaces, the objective is to find $k$ linear (or affine) subspaces of dimension at most $q$ (where $q < d$) that minimize the average distance of any data point to its nearest subspace. We consider the case where the subspaces are linear subspaces. Here the distances from a point to the linear subspaces are the lengths of orthogonal projections (in the $\ell_2$ norm) of the points in linear subspace, to the power $z$, for $z \ge 1$.
This then gives the linear subspace clustering problem.
\begin{definition}[\textbf{$(q,k,z)$ linear subspace clustering}]
\label{def:qkz_clustering}
    Given a set $X$ of $n$ points in $\R^d$, a weight function $w :  X \rightarrow \R_{>0}$, number of clusters $k$, the dimension $q \le d$, and a real number $z \in [1, \infty)$, the objective of $(q,k,z)$ linear subspace clustering is to find $k$ linear subspaces $V = \paren{V_1, V_2, \ldots, V_k}$, where each $V_i$ is of dimension at most $q$, so as to minimize the cost of clustering defined as
	\[
	\cost(V,X,w) := \left(\sum_{x \in X} w(x)\cdot d(V,x)^z\right)^{1/z},\quad
	\text{where}~~d(V,x):=\min_{i\in [k]} d(V_i,x) = \|x^TZ_i\|_2.
	\]
\end{definition}
where $Z_i$ is orthogonal projection matrix corresponding to $V_i$, for every $i \in [k]$.
The $k=1$ case corresponds to dimension reduction by fitting a low-dimensional subspace to the given data, popularly done via Principal Component Analysis (PCA).
Even in this case we consider uniform weights.

As large data sets are often high-dimensional and also have a large number of points, a popular technique relevant to our paper is to construct a small \emph{coreset} of the given data and perform clustering on the coreset. 
A coreset is a small weighted sample of the data such that clustering cost for this smaller set of points gives a $(1+\varepsilon)$ approximation to the clustering cost for the entire data. 
Following the unified framework for clustering using coresets proposed in \cite{FL2011}, recent works \cite{HV2020, CSS2021} have constructed coresets for $k$-means or $k$-medians clustering of size {\sf poly}$\left(k, 1/\varepsilon\right)$, independent of $n$ and $d$.

In many clustering applications the points being clustered are from different socially salient groups (eg., based on race, gender, age, education). Let us assume that we also have this group information of the points.
That is, the set of points can be partitioned into $\ell$ disjoint groups. 
Let us denote the partitions of $X$ based on these groups as $X_1, \ldots, X_{\ell}$.
Then the clustering objective defined above may result in harmful treatment of some of groups, as captured by the following notions of fairness in clustering.

\paragraph{Fair Clustering.}
To alleviate representational harms to different demographic groups in a cluster, a natural objective for fair clustering is to ensure that in each cluster, the proportion of any group in that cluster should be nearly the same as its proportion in the overall population \cite{CKLV2017}. Subsequent work has generalized this to clustering objectives with upper and lower bound constraints on the group-wise representation in each cluster \cite{bera2019fair,Schmidt2019CoresetsFC}. We do not study the above objective but since our work is closely related to coresets, we note that for the above fair $k$-means and fair $k$-median objectives, a coreset of size $\mathcal{O}\left(\Gamma\varepsilon^{-d}k^2\right)$ can be constructed by a deterministic algorithm \cite{HJV2019}, and a coreset of size {\sf poly}$\left(\Gamma, \log n, k, 1/\varepsilon\right)$ can be constructed by a sampling-based algorithm \cite{BFS2020}, where $\Gamma$ is the number of types of items (i.e., number of disjoint groups). It is important to keep in mind that the above fair clustering objective can be inadequate when the benefits or harms of clustering for different demographic groups are better represented by their clustering cost rather than their proportional representation \cite{GSV2021}. 

\paragraph{Socially Fair Clustering.}
To alleviate quality-of-service harms to different demographic groups, it is natural to define fairness in terms of the group-wise clustering cost. In \textit{socially fair} clustering, the goal is to minimize the maximum average group-wise cost.
Formally we have the following socially fair variant of \Cref{def:kz_clustering}.
\begin{definition}[\textbf{socially fair $(k,z)$ clustering}]
\label{def:sfair_kz}
	Given a set $X$ of $n$ points in $\R^d$ divided into $\ell$ disjoint groups \\$(X_1,\ldots,X_\ell)$, a positive integer $k$ and a real number $z \in [1, \infty)$, the objective of socially fair $(k, z)$ clustering is to find $k$ points $C = \paren{c_{1}, c_{2}, \dotsc, c_{k}}$, with each $c_{i} \in \R^d$, that minimize the maximum per-point clustering cost, $\cost(C, X_{j})$, for clustering any group $X_{j}$ using $C$.
	We represent socially fair $(k,z)$ clustering cost as
	\[
	\fcost(C,X) := \max_{j \in [\ell]} \cost(C, X_j ).
	\]
	We call $(c_1, c_2, \dots, c_k)$ the socially fair centers.
\end{definition}

Socially fair $k$-means clustering has been introduced and studied in \cite{GSV2021} and \cite{ABV2021} simultaneously.
\cite{ABV2021} provide an $\cO\paren{\ell}$ approximation algorithms, whereas \cite{GSV2021} give a Lloyd-like heuristic algorithm that performs well in practice.
Later \cite{MV2021} presented a polynomial time $\paren{e^{\cO\paren{z}}\frac{\log \ell}{\log \log \ell}}$-approximation algorithm for the more general socially fair $(k,z)$ clustering that we study.
\cite{CMV2022} study a more general fair clustering objective where the objective is to minimize the $\ell_p$ norm of the average group-wise clustering costs.
When $p = \infty$ this is nothing but the socially fair clustering cost.
They also achieve the same bound as \cite{MV2021}.
Recently \cite{GJ2021} came up with $(33+\varepsilon)$ and $(5 + \varepsilon)$ approximation algorithms for socially fair $k$-means and $k$-medians clustering respectively, which run in time $\paren{k/\varepsilon}^{\cO(k)}n^{\cO(1)}$.
In an independent and concurrent work \cite{GSV2022} came up with a $(5 + 2\sqrt{6}+\varepsilon)$ approximation algorithm that runs in time $n^{2^{O(z)}\cdot\ell^2}$ and a $(15+6\sqrt{6})$ approximation algorithm that runs in time $k^\ell\cdot {\sf poly}(n)$ for the socially fair $(k,z)$ clustering problem.

We also have the following socially fair variant of the linear subspace clustering.
\begin{definition}[\textbf{socially fair $(q,k,z)$ linear subspace clustering}]
\label{def:sfair_qkz}
	Given a set $X$ of $n$ points in $\R^d$ divided into $\ell$ disjoint groups $(X_1,\ldots,X_\ell)$, a positive integer $k$, the dimension $q < d$, and a real number $z \in [1, \infty)$, the goal of socially fair $(q,k,z)$ linear subspace clustering is to find $k$ linear subspaces $V  = \paren{V_1, V_2, \ldots, V_k}$ of dimension at most $q$ in $\R^d$ so as to minimize the maximum per-point clustering cost, $\cost(V, X_j)$, for clustering any group $X_j$ using $V$.
	We represent socially fair $(q,k,z)$ linear subspace clustering cost as
	\[
	\fcost(V,X) := \max_{j \in [\ell]} \cost(V, X_j).
	\]
	We call $(V_1, V_2, \dots, V_k)$ the socially fair linear subspaces.
\end{definition}
In case of socially fair $(q,k,z)$ linear subspace clustering with $k = 1$, the optimal solution can be obtained efficiently by Principal Component Analysis (PCA), but applying PCA on the data in a group-blind fashion might result in inequitable group-wise costs, as observed in \cite{STMSV2018}. 
We study the objective of socially fair linear subspace clustering using $k$ subspaces of dimension $q$ each so as to minimize the maximum average group-wise clustering cost defined using $z$-th power of distances.
In \cite{GSV2021}, the authors study a slightly different problem, where the goal is to find optimal subspace such that among all the groups, the increase in the per point cost of the group due to fairness constraints compared to PCA of that group is minimized.
To the best of our knowledge, we are the first to study the socially fair linear subspace clustering problem.

\paragraph{Our Contributions.} The main contributions of our paper can be summarized as follows,

\begin{itemize}[leftmargin=*]
    \item We propose a generic framework (\Cref{alg}), where we first assume that there is an oracle that, given a $k$-partitioning of the data, outputs $k$ centers (or linear subspaces) such that the socially fair center-based (or linear subspace) clustering cost with respect to these centers (subspaces) is at most $\alpha$ times the optimal socially fair clustering cost for this $k$-partitioning. 
	Then our framework gives an $\alpha(1+\varepsilon)$-approximation algorithm to the socially fair clustering problem, given access to an appropriate strong $\varepsilon$-coreset construction algorithm and the $\alpha$ approximate oracle (\Cref{thm:framework}).
	\item If the best known strong coreset construction algorithm for an unconstrained clustering problem outputs a coreset $S$, we show that we can construct a strong coreset of size $\ell\cdot |S|$ for both socially fair clustering. Here $\ell$ is the number of groups the points belong to (\Cref{thm:coreset}).
	\item We also give constructions of efficient oracles, which when used in our framework give a $(1+\varepsilon)$ approximation algorithm for socially fair center-based clustering (\Cref{thm:kz}) and a $\sqrt{2}\ell^{\frac{1}{z}}\gamma_z(1+\varepsilon)$ approximation algorithm for socially fair linear subspace clustering, for $z \geq 2$, where $\gamma_z \approx \sqrt{z/e}(1+o(1))$ (\Cref{thm:qkz}). 
	\item We also propose a Lloyd-like heuristic algorithm (\Cref{alg:algo2})  to perform socially fair clustering efficiently for  practical purposes.
	\item We show experimental results on many benchmark datasets such as Credit Card, Adult Income, German Credit, Bank, and Skill Craft datasets and compare the results with unconstrained as well as fair baselines wherever applicable (\Cref{sec:experiments}).
\end{itemize}

\section{Framework}\label{sec:framework}
Recall that we are given a set $X$ of $n$ points in $\R^d$. 
A coreset is a small weighted sample of the data that closely approximates the clustering cost of the data.
Hence, clustering algorithms run much faster on coresets.
In this section, we present two of our main results.
We first show that we can construct a coreset for the socially fair clustering objective, given coresets for the unconstrained clustering objective for each of the groups.
We remark that previous works have constructed coresets for other notions of fairness \cite{HJV2019, BFS2020} in clustering but not for socially fair clustering. 
Our second result is a framework 
(see \Cref{alg}) that can be used to solve both variants of the socially fair clustering problems described in \Cref{sec:problem}.
More formally, consider the following definition of coresets for clustering.

\begin{definition}[\textbf{strong coresets for $(k,z)$ clustering}]
	Given a set $X$ of $n$ points and a constant $\varepsilon > 0$ with weight function $w : X \rightarrow \R_{> 0}$, a weighted sample $S$ of the points with weight function $\hat{w} : S \rightarrow \R_{> 0}$ is a strong $\varepsilon$-coreset of $X$ for the $(k,z)$ clustering problem if for every  $k$ centers $C = \paren{c_1, \ldots, c_k}$ such that $c_i \in \R^d, \forall i \in [k]$, 
	\[
	\cost(C,S,\hat{w}) \in (1\pm \varepsilon) \cdot\cost(C, X, w).
	\]
\end{definition}

\begin{definition}[\textbf{strong coresets for $(q,k,z)$ linear subspace clustering}]
	Given a set $X$ of $n$ points and a constant $\varepsilon > 0$ with weight function $w : X \rightarrow \R_{> 0}$, a weighted sample $S$ of the points with weight function $\hat{w} : S \rightarrow \R_{> 0}$ is a strong $\varepsilon$-coreset for the $(q,k,z)$ linear subspace clustering problem if for every $k$ linear subspaces $V = \paren{V_1, V_2, \ldots, V_k}$, where each $V_i$ is a linear subspace in $\R^d$ of dimension at most $q$,
	\[
	\cost(V,S,\hat{w}) \in (1\pm \varepsilon) \cdot\cost(V, X, w).
	\]
\end{definition}

Henceforth, whenever we say coreset we refer to a strong coreset. 
Recall that we are additionally given that the points belong to $\ell$ disjoint groups.
Given the partition of the data based on these groups, $X_1, \ldots, X_{\ell}$, the key observation we first make is that a union of coresets for each of the groups for the unconstrained clustering cost is also a coreset for the entire data for the socially fair clustering cost.
This is true both for socially fair $(k,z)$ clustering as well as socially fair $(q,k,z)$ linear subspace clustering. Note that a union of coresets for groups has been used as a coreset for other variants of fair clustering (e.g., \cite{HJV2019}, for the problem of $(\alpha, \beta)$-proportionally-fair clustering).
\begin{theorem}[coresets for socially fair clustering]
\label{thm:coreset}
	Let $S_j$ be a strong $\varepsilon$-coreset for $X_j$, for each $j \in [\ell]$, with respect to the clustering cost. Then, $S = \bigcup_{j=1}^{\ell} S_j$ is a strong $\varepsilon$-coreset for the entire data $X$ with respect to the socially fair clustering cost.
\end{theorem}

\begin{proof}[Proof of \Cref{thm:coreset}]
	Let $S_j$ be a strong $\varepsilon$-coreset for the points belonging to group $j \in [\ell]$ with respect to the clustering cost. Then for any $k$ centers $C = \paren{c_1, \ldots, c_k}$, let $X_1, X_2, \ldots, X_k$ be clustering obtained by assigning the points in $X$ to their closest centers in $C$. Then we have,
	\[
	(1-\varepsilon)\cdot \cost(C,X_j) \leq \cost(C, S_j) \leq (1+\varepsilon)\cdot\cost(C, X_{j}). 
	\]
	Now let $S = \bigcup_{j=1}^{\ell} S_j$.
	For the same cluster centers and  partitioning of the data we have that,
	\begin{align*}
	\fcost(C,S) = \max_{j \in [\ell]} &\cost(C, S_j) \\
	\implies \max_{j \in [\ell]} (1-\varepsilon)\cdot \cost(C, X_j) \le \fcost&(C,S) \le \max_{j \in [\ell]} (1+\varepsilon)\cdot \cost(C, X_j) \quad\because S_j\text{ is an $\epsilon$ coreset of }X_j\\
	\implies (1-\varepsilon)\cdot\max_{j \in [\ell]}  \cost(C, X_j)  \le \fcost&(C,S) \le (1+\varepsilon)\cdot\max_{j \in [\ell]}  \cost(C, X_j) \\
	\implies (1-\varepsilon)\cdot \fcost(C, X)\le \fcost&(C,S) \le (1+\varepsilon)\cdot \fcost(C, X).
	\end{align*}
	
	Therefore $S$ is an $\varepsilon$-coreset for the whole dataset w.r.t. the socially fair clustering cost. The same arguments also work for socially fair $(q,k,z)$ linear subspace clustering.
\end{proof}
\begin{algorithm}[t]
\caption{Framework for the  socially fair clustering}
\label{alg}
\begin{algorithmic}[1]
\Require The set of points $X$ and their group memberships, the numbers $k$, $z$.

\Ensure The cluster centers $C$.

\State Compute $S_j$, a strong $\varepsilon/3$-coreset for group $j, \forall j\in[\ell]$.
\State Let $S := \bigcup_{j\in[\ell]} S_j$.
\For {each $k$ partitioning of the coreset $\cP(S, k) := \paren{P_1(S), P_2(S), \ldots, P_k(S)}$}
\State $C := \text\oracle(\cP(S, k))$.
\State $t := \max_{j \in [\ell]} \cost(C,S_j)$.
\EndFor
\State \Return The centers $C$ with minimum value of $t$.
\end{algorithmic}
\end{algorithm}

As a consequence of \Cref{thm:coreset}, we propose an iterative framework to solve socially fair clustering (see \Cref{alg}).
The framework crucially depends on the existence of an efficient algorithm called \oracle, that, given a set of clusters, outputs $\alpha$-approximate socially fair centers. 
Then the socially fair clustering cost output by this function is at most $\alpha$ times the optimal socially fair clustering cost, where $\alpha \ge 1$.
A similar framework can also be used for socially fair $(q,k,z)$ linear subspace clustering (see Algorithm~\ref{framework_subspace} in the appendix).

\paragraph{Overview of the framework.}
Let us assume the centers $C$ here are the points in $\R^d$. The following arguments also hold for the centers that are linear subspaces.
Let us denote the partitions of $X$ based on the group information as $X_1, \ldots, X_{\ell}$. 
We first compute a strong coreset $S_j$ for each of the groups $j \in [\ell]$, corresponding to the $(k,z)$ clustering cost.
Then by \Cref{thm:coreset}, $S = \cup_{j \in [\ell]} S_j$ gives us a strong coreset for the data for socially fair $(k,z)$ clustering cost.
Then let $\cP(S, k) = \paren{P_1(S), P_2(S), \ldots, P_k(S)}$ denote a partitioning of the points $S$ into $k$ clusters and let $S_{ij}$ denote the set of points belonging to cluster $i$ and group $j$.
That is, $S_{ij} = P_i(S \cap X_j)$ for all $i \in [k]$ and $j \in [\ell]$.
We then iterate over all possible $k$-clusters ($k$-partitioning) of the coreset $S$.
In each iteration we call the algorithm \oracle~with the current clusters of $S$.
It gives $\alpha$-approximate socially fair centers for the coreset $S$ (we show an efficient implementation of \oracle~for different clustering objectives in \Cref{sec:cvxprog}).
We return the socially fair cluster centers $C$ with the minimum $\fcost(C,S)$.
If \oracle~ has a running time $t\paren{S, \alpha}$, the running time of this framework is $\cO\paren{k^{|S|}\cdot t\paren{S, \alpha}}$.
Therefore, we have the following result stated for socially fair center-based clustering.

\begin{theorem}
\label{thm:framework}
	For a set $X$ of $n$ points in $\R^d$ belonging to $\ell$ disjoint groups, $X_1,\ldots,X_{\ell} \subseteq X$, let $S_j$ be a strong $\varepsilon/3$-coreset for group $j$ for some $\varepsilon \in [0,1]$, $\forall j \in [\ell]$. Let $t\paren{S, \alpha}$ be the time taken by \oracle~that gives an $\alpha$-approximate socially fair cluster centers, given clusters. Then there exists an $\alpha(1+\varepsilon)$ approximation algorithm for the socially fair clustering problem that runs in time $\cO\paren{k^{|S|}\cdot t\paren{S, \alpha}}$ where $S = \bigcup_{j=1}^{\ell} S_j$. 
\end{theorem}
\begin{proof}[Proof of \Cref{thm:framework}]
We use the framework given in \Cref{alg} with the appropriate \oracle~function for the socially fair clustering objective.
Let $C^*$ be the optimal centers for clustering the data $X$, then,
\begin{equation}
\label{opt:data}
    \fcost(C^*,X) \le \fcost(C,X), \quad \forall C \in \R^{k\times d}.
\end{equation}
Similarly, let 
$\widehat{C}$ be the optimal centers for clustering the coreset $S$, then,
\begin{equation}
\label{opt:coreset}
    \fcost(\widehat{C},S) \le \fcost(C,S), \quad \forall C \in \R^{k\times d}.
\end{equation}
Let $\overline{C}$ be the centers returned by \Cref{alg}, and $C'$ be the centers given by \oracle~for the optimal clustering on $S$.
Then,
\begin{equation}
\label{eq:oracle}
    \fcost(\overline{C},S) \le \fcost(C',S) \le \alpha\cdot \fcost(\widehat{C},S).
\end{equation}
Consequently, $\fcost(\overline{C},X)  $
\begin{align*}
    &\le \frac{1}{1-\varepsilon/3}\fcost(\overline{C},S) &(\text{from }\Cref{thm:coreset})\\
    &\le \frac{\alpha}{1-\varepsilon/3}\fcost(\widehat{C},S) &(\text{from }\Cref{eq:oracle})\\
    &\le \frac{\alpha}{1-\varepsilon/3}\fcost(C^*,S)
    &(\text{from }\Cref{opt:coreset})\\
    &\le \frac{\alpha(1+\varepsilon/3)}{1-\varepsilon/3}\fcost(C^*,X)
    &(\text{from }\Cref{thm:coreset})\\
    &\le \alpha(1+\varepsilon)\fcost(C^*,X).
\end{align*}

Moreover, the running time of the algorithm is $\cO\paren{k^{|S|} \cdot t\paren{S, \alpha}}$ since we are iterating over all possible $k$-partitioning of the coreset $S$ and each iteration calls \oracle~once, which takes time $t\paren{S, \alpha}$.
\end{proof}

\section{Implementation of \oracle}\label{sec:cvxprog}

In this section, we give an exact implementation of \oracle~for socially fair $(k,z)$ clustering and an approximate one for socially fair $(q,k,z)$ linear subspace clustering. We do both by formulating them as convex programs that can be solved efficiently using the well known Ellipsoid method.
Note that in \Cref{alg:algo2} we call the function \oracle~for the 
$\varepsilon$-coreset $S$ corresponding to the given dataset $X$ for the socially fair clustering objective.
We show the implementation of the \oracle~for $S$ but it works for any subset of $X$.

\paragraph{Socially Fair $(k,z)$-Clustering.}
Recall that $S_{i j} = P_i(S) \cap S_j$.
Then the problem of finding $k$ centers that minimize the socially fair $(k,z)$ clustering cost, given the clusters, can be expressed as the following convex program, 
\begin{equation}
    \min_{\beta\in \mathbb{R}, c_1,\cdots,c_k \in \mathbb{R}^d} \:\:\:\: \beta,\label{eq:cvxprog1}
\end{equation}
\begin{equation*}
    \text{such that}\quad  \frac{1}{|S_j|}\sum\limits_{i \in [k]} \sum_{x \in S_{i j}} \|x-c_i\|^z \le \beta, \:\:\:\: \forall j \in [\ell].
\end{equation*}
It is easy to see that the optimal solution to this convex program $\beta^*, c_1^*, c_2^*, \ldots, c_k^*$ gives socially fair clustering cost $\beta^*$, with the centers $\paren{c_1^*, c_2^*, \ldots, c_k^*}$. Let $L$ be the bit complexity of the input, that is, each point $x$ is a vector in $\Q^d$ and all the rational numbers are represented as fractions where the values of numerator and denominator can be represented using at most $L$ bits. 
\begin{theorem}
\label{thm:kz}
    Given a set of $n$ points in $\R^d$, with bit complexity $L$, that are partitioned into $\ell$ groups, numbers $k \in \Z_{\ge 0}$, $z \ge 1$, an algorithm with running time $T_S$ to compute a strong $\varepsilon$-coreset $S$ for socially fair $(k,z)$ clustering, there exists a $(1+\varepsilon)$-approximation algorithm to socially fair $(k,z)$ clustering that runs in time
    \[
    \widetilde{\cO}\paren{k^{|S|}\cdot\paren{k^2d^2+\ell kL^2(d+z)|S|}\cdot k^2d^2\cdot \paren{Lz+d+\log|S|}^2 + T_S}.
    \]
\end{theorem}
To the best of our knowledge, the best known coreset size for the $(k,z)$ clustering is, $|S| = \widetilde{\mathcal{O}}\paren{\text{min}\{\varepsilon^{-2z-2},2^{2z}\varepsilon^{-4}k\}\ell k}$ by \cite{HV2020}. 
The algorithm to construct this coreset runs in time $\widetilde{\mathcal{O}}\paren{ndk}$.
Moreover, \cite{MMI2019} show that the cost of any clustering is preserved upto a factor of $(1+\varepsilon)$ under a projection onto a random
$\cO(\log(k/\varepsilon)/\varepsilon^2
)$-dimensional subspace.
Therefore, the running time of our algorithm is $\widetilde{\cO}(ndk) + 2^{\cO(\varepsilon^{-z}k^3\ell \log L )}$.
Note that even for unconstrained $k$-means clustering problem, the best known $(1+\varepsilon)$ approximation algorithm runs in time exponential in $k$ (see Theorem 1 in \cite{BJK2018}).

We use the Ellipsoid method to solve the convex program in \ref{eq:cvxprog1}.
The Ellipsoid method gives the following,
\begin{theorem}[Theorem 13.1 in \cite{vishnoi_algorithms_for_convex_opt2021}]
\label{thm:vishnoi}
For a convex program of the form, $min_{x\in K} f(x)$, where $f$ is a convex function and $K\subseteq  \R^m$ is a convex set,
there is an algorithm that, given a first-order oracle for a convex function $f : \R^m \rightarrow \R$ with running time $T_f$, a separation oracle for a convex set $K$ with running time $T_{\phi}$, numbers $r > 0$ and $R > 0$ such that $K \subseteq B(0,R)$ and $K$ contains a Euclidean ball of radius $r$, bounds $l_0$ and $u_0$ such that $\forall x\in K,l_0 \leq  f(x) \leq u_0$, and an $\varepsilon >0$,
outputs a point $\widehat{x} \in K$ such that
$
f(\widehat{x})\leq  f(x^*)+\varepsilon,
$
where $x^*$ is any minimizer of $f$ over $K$. 
The running time of the algorithm is 
\[
\cO\paren{\paren{m^2+T_{\phi}+T_f}\cdot m^2\cdot \log^2\paren{\frac{R}{r} \cdot \frac{u_0 - l_0}{\varepsilon}}}.
\]
\end{theorem}
\paragraph{Using the Ellipsoid method in \Cref{thm:vishnoi}.} 
The convex program in (\ref{eq:cvxprog1}) has a linear objective function.
Let 
\[C := (c_1, c_2, \ldots, c_k)~~\text{and}~~f_j(C) := \frac{1}{|S_j|}\sum\limits_{i \in [k]} \sum_{x \in S_{i j}} \|x-c_i\|^z.
\]

The algorithm does a binary search for the optimal value $\beta$ that is between $l_0$ and $u_0$.
For a fixed value of $\beta$, 
$
\widehat{K} := \set{C | f_j(C) \le \beta, \forall j \in [\ell]},
$
is also a convex set since each set $C$ that satisfies $f_j(C) \le \beta$ is a convex set and intersection of convex sets is also a convex set (see Section 2.3.1 in \cite{boyd2004convex}).
In each step, we have a separation oracle for the convex set $\widehat{K}$ and hence, can use the ellipsoid method to check for the feasibility of $\widehat{K}$.
At the end of the binary search, we have a solution with value at most $\beta^*+\varepsilon$. 
We can then query the separation oracle with this value to get a point in $\widehat{K}$ with value at most $\beta^*+\varepsilon$.

\begin{lemma}
\label{lem:seperation_oracle}
    There is a separation oracle and a first-order oracle for the convex program in (\ref{eq:cvxprog1}) for a fixed $\beta$, and each take time $\cO\paren{\ell kL^2(d+z)|S|}$.
\end{lemma}
\begin{proof}
Let $\beta$ be the value in the current iteration of the binary search.
For $j \in [\ell]$, let $F_j(C) := f_j(C) - \beta$.
Then $F_j(C)$ is also a convex function.
Therefore, the following always holds,
\begin{equation}\label{eqn:cvx_func_property}
    F_j(C) \ge F_j(C') + \langle \nabla F_j(C'), C-C' \rangle, \forall C, C' \in \widehat{K}.
\end{equation}
Then for any point $C \in \widehat{K}$, the first order oracle outputs $F_j(C)$ and $\paren{\nabla_1 F_j(c_1), \nabla_2 F_j(c_1), \ldots, \nabla_k F_j(c_k)}$ in time $\cO\paren{kL^2(d+z)|S_{j}|}$ where $L$ is the bit complexity of the problem and $\nabla_i F_j(c_i)$ is the derivative of the function $F_j$ with respect to the point $c_i$, for any $i \in [k]$.

For an input $C' \in \widehat{K}$, the separation oracle should output one of the following,
\begin{itemize}
    \item If $C' \in \widehat{K}$, output YES.
    \item If $C' \notin \widehat{K}$, output NO and a hyperplane defined by $h (\ne 0) \in \widehat{K}$ such that $\forall C \in \widehat{K}$, $\langle h,C \rangle < \langle h,C' \rangle$.
\end{itemize}
The first case is easy to check in polynomial time. 
So let us assume $C' \notin \widehat{K}$.
Then $\exists j \in [\ell]$ such that $F_j(C') > 0$. 
This implies that $F_j(C') > F_j(C), \forall C\in \widehat{K}$. 
Substituting this in equation \ref{eqn:cvx_func_property} gives us, $\langle \nabla F_j(C'), C-C' \rangle < 0$. Therefore, $\langle \nabla F_j(C'),C \rangle < \langle \nabla F_j(C'),C' \rangle , \forall C \in \widehat{K}$. 
Since we have access to the first order oracle, the separation oracle also runs efficiently.

Note that if $\nabla F_j(C') = 0$, $F_j(C') \leq F_j(C)$, $\forall C \in \widehat{K}$ since $F_j$ is convex. In which case $F_j(C') > 0 \implies F_j(C) > 0$ for all $C \in \widehat{K}$. If this holds $\widehat{K}$ is empty.
Hence the binary search proceeds to the next step.

Therefore, the separation oracle also has the same running time as the first order oracle.
Since we need to check $\ell$ of them, the lemma follows.
\end{proof}

\begin{proof}[Proof of \Cref{thm:kz}]

The obvious lower bound estimate $l_0$ is $0$. For the upper bound estimate $u_0$, we can consider the value $|S|\cdot \text{max}_{x,y}\|x-y\|^z$, which is $2^{\mathcal{O}(Lz+d+\log|S|)}$. Consequently, the radius of the outer ball $R$ will also be $2^{\mathcal{O}\paren{Lz+d+\log|S|}}$ since we can always take $R$ to be $2u_0$. 
Since $\widehat{K}$ is full dimensional, $r$ will be $2^{-\cO\paren{L}}$.
Therefore, the result follows from \Cref{thm:framework}, \Cref{thm:vishnoi}, and \Cref{lem:seperation_oracle}.
\end{proof}

\paragraph{Socially Fair $(q,k,z)$ Linear Subspace Clustering.}
Let $\set{V_i}_{i \in [k]}$ denote $k$ socially fair linear subspaces, each of rank at most $q$. 
Let $z_{i,1}, z_{i,2}, \cdots, z_{i,d-q}$ denote the orthonormal basis of the orthogonal complement of $V_i$ and let $Z_i \in \mathbb{R}^{d\times d-q}$ denote the matrix with the $h^{\text{th}}$ column as $z_{i,h}$. We know that the distance of a point $x$ to the $i$th subspace is its projection on the orthogonal subspace. Thus, 
$
d(x,V_i) = \|x^TZ_i\|_2.
$
Here let $S$ be an $\epsilon$-coreset for $X$ with respect to the socially fair $(q,k,z)$ linear subspace clustering objective.
Then, the oracle can be expressed as follows,
\begin{align}\label{cvxprog:linprojclus}
    &\min_{\beta\in \R, Z_1,\cdots,Z_k \in \mathbb{R}^{d\times d-q}} \quad \beta,
\end{align}
\begin{align*}
    \text{s.t.,}\quad & \frac{1}{|S_j|}\sum\limits_{i \in [k]} \sum_{x \in S_{i j}} \|x^TZ_i\|_2^z \le \beta, ~~\forall j \in [\ell],\\
    &\|Z_i^{(h)}\|_2 \ge 1, ~~ \forall i \in [k],\forall h\in [d-q],\\
    &\langle Z_i^{(h_1)},Z_i^{(h_2)} \rangle = 0,~~ \forall h_1 \ne h_2, i \in [k].
\end{align*}

This problem can be approximated by utilizing techniques from ~\cite{DTV2011}. More precisely, we have the following theorem,

\begin{theorem}
\label{thm:qkz}
    Given a set of $n$ points in $\R^d$, with bit complexity $L$, that are partitioned into $\ell$ groups, numbers $k,q \in \Z_{\ge 0}$, $z \ge 1$, and an algorithm with running time $T_S$ to compute a strong $\varepsilon$-coreset $S$ for socially fair $(q,k,z)$ linear subspace clustering, there exists a  $\sqrt{2}\ell^{\frac{1}{z}}\gamma_z(1+\varepsilon)$ approximation algorithm to socially fair $(q,k,z)$ linear subspace clustering that runs in time\\ $$\widetilde{\cO}\paren{k^{|S|}\cdot\paren{k^2d^4+\ell kd^2zL|S|+\ell d^3L}\cdot k^2d^4\cdot \paren{Lz+d+\log|S|}^2+T_S}.$$
\end{theorem}

In a recent work \cite{TWZ+22} that studies the problem of constructing a strong coreset for $(q,k,z)$ projective clustering when the points come from a polynomial grid, the authors provide a constant-factor approximation coreset of size
$\cO\paren{\paren{8q^3\log(d\Delta)}^{O(qk)}}$ for $z = \infty$, with running time $\cO\paren{nq^4\paren{\log\Delta}^{q^2k}}$, and an $\varepsilon$-coreset of size $\cO\paren{\paren{8q^3\log(d\Delta)}^{O(qk)}
\log n}$ for $z=2$ with running time $\cO\paren{n^2q^4\paren{\log\Delta}^{q^2k}}$, where $\Delta$ is the ratio of the largest and the smallest coordinate magnitudes.
Substituting this we get that our algorithm runs in time 
$\cO\paren{n^2q^4\paren{\log\Delta}^{q^2k}} + n\cdot 2^{\cO((q\log(d\Delta)^{\cO(qk)}\log L)}$ for socially fair $(q,k,2)$ linear subspace clustering.

\begin{proof}[Proof of \Cref{thm:qkz}]
We consider the convex relaxation of \ref{cvxprog:linprojclus} similar to~\cite{DTV2011},

\begin{align*}
    &\min_{\beta\in \R, \widehat{Z}_1,\cdots,\widehat{Z}_k \in \mathbb{R}^{d\times d}} \quad \beta,\\
    \text{s.t.,}\quad &\sum\limits_{i \in [k]} \sum_{x \in S_{i j}} |x^T\widehat{Z}_ix|^{z} \le \beta, ~~\forall j \in [\ell],\\
    &\Tr\paren{\widehat{Z}_i} \ge d-q, ~~ \forall i \in [k],\\
    &I\succcurlyeq \widehat{Z}_i \succcurlyeq 0, ~~ \forall i \in [k].
\end{align*}
In this case, let 
\[
\widehat{Z} := \paren{\widehat{Z}_1, \widehat{Z}_2, \ldots, \widehat{Z}_k}~~\text{and}~~F_j\paren{\widehat{Z}} := \sum\limits_{i \in [k]} \sum_{x \in S_{i j}} |x^T\widehat{Z}_ix|^{z} - \beta, \forall j \in [\ell].
\]
Then the separation oracle for the first set of constraints $F_j\paren{\widehat{Z}} \le 0, \forall j \in [\ell]$ can be constructed in exactly the same way as in \Cref{lem:seperation_oracle}.

Therefore, we now show the separation oracle for the rest of the constraints.
Note that the trace constraint is a linear constraint.
Therefore, we can use $h\paren{\widehat{Z}_i} = d-q-\Tr(\widehat{Z}_i)$ as a separating hyperplane.
For the constraint $\widehat{Z}_i' \succcurlyeq 0$, the following cases arise,
\paragraph{Case A:}If $\widehat{Z}_i'$ is symmetric, then due to the spectral decomposition, we have all the eigenvectors and eigenvalues of the matrix $\widehat{Z}_i'$.
    If all the eigenvalues are positive, then it satisfies the constraint, otherwise output the hyperplane $-vv^T$ where $v$ is the eigenvector corresponding to one of the negative eigenvalues (say $\lambda$).
    Then,
    \begin{align*}
    \langle -vv^T, \widehat{Z}_i\rangle = -v^T\widehat{Z}_iv \leq 0< -\lambda v^Tv = v^T\widehat{Z}_i'v = \langle -vv^T, \widehat{Z}_i'\rangle, \quad \forall \widehat{Z}_i, \widehat{Z}_i' \in \Q^{d\times d}.
    \end{align*}
\paragraph{Case B:}If $\widehat{Z}_i'$ is not symmetric, then it does not satisfy the constraint.
Let $r,s \in [d]$ be such that $\widehat{Z}_i'(r,s) > \widehat{Z}_i'(s,r)$ w.l.o.g.
Then output a matrix $M$ where $M(r,s) = 1$, $M(s,r) = -1$, and all other entries are $0$.
Then for any symmetric matrix $\widehat{Z}_i$ we have that 
\[\langle M, \widehat{Z}_i\rangle = 0 < \widehat{Z}_i'(r,s) - \widehat{Z}_i'(s,r) = \langle M, \widehat{Z}_i'\rangle.\]
Since the set of positive semi-definite matrices is a subset of the set of symmetric matrices, and that the set of symmetric matrices is also a convex set, this forms a valid separation oracle. This separation oracle runs in time $\cO\paren{d^{3}L}$.

The constraint $I \succcurlyeq \widehat{Z}_i'$ can be written as $I - \widehat{Z}_i' \succcurlyeq 0$.
Then checking positive semideifiniteness if $I-\widehat{Z}_i'$ is symmetric is the same as case A above.
For case B, that is when $I - \widehat{Z}_i'$ is not symmetric, let $r,s \in [d]$ be such that $(I-\widehat{Z}_i')(r,s) > (I-\widehat{Z}_i')(s,r)$ w.l.o.g.
Then output a matrix $M$ where $M(r,s) = -1$, $M(s,r) = 1$, and all other entries are $0$.
Then for any symmetric matrix $\widehat{Z}_i$ we have that 
\[\langle M, I-\widehat{Z}_i\rangle = 0 < \widehat{Z}_i'(r,s) - \widehat{Z}_i'(s,r) = \langle M, I-\widehat{Z}_i'\rangle.\]
This separation oracle also runs in time $\cO\paren{d^{3}L}$.

Even in this case the obvious lower bound is $l_0 = 0$ and since $I \succcurlyeq \widehat{Z}_i, \forall i \in [k]$, the upper bound is $u_0 = |S|\cdot \max_x |x^T I x| = |S|\cdot \max_x \|x\|^2$ which is $2^{\mathcal{O}(Lz+d+\log|S|)}$. Consequently, the radius of the outer ball $R$ will also be $2^{\mathcal{O}\paren{Lz+d+\log|S|}}$ since we can always take $R$ to be $u_0$. 
Since $K$ is full dimensional, $r$ will be $2^{-\mathcal{O}\paren{L}}$.

After solving the convex program, we can apply the rounding algorithm of~\cite{DTV2011} independently for each of the $k$ linear subspaces, giving us $(V_1,\cdots,V_k)$ that have a $\sqrt{2}\gamma_z$ approximation factor for each group. Then,

\begin{align*}
    \E[\fcost(V,S)] & = \E[\max_{j\in [\ell]}\cost(V,S_j)]\\
     & = \E[(\max_{j\in [\ell]}\cost^z(V,S_j))^{1/z}]\\
     & \le \paren{\E[\max_{j\in [\ell]}\cost^z(V,S_j)]}^{1/z} & (\text{by Jensen's inequality})\\
     & \le \big(\E\big[\sum_{j\in [\ell]}\cost^z(V,S_j)\big]\big)^{1/z}\\
     & = \big(\sum_{j\in [\ell]}\E[\cost^z(V,S_j)]\big)^{1/z}&(\text{by linearity of expectation})\\
     & \le \big(\ell\cdot \max_{j\in [\ell]}\E[\cost^z(V,S_j)]\big)^{1/z}\\
     & = \ell^\frac{1}{z}\cdot\max_{j\in [\ell]}(\E[\cost^z(V,S_j)])^{1/z}\\
     & \le \ell^\frac{1}{z}\sqrt{2}\gamma_z\cdot\opt & (\text{from}~\cite{DTV2011}),
\end{align*}
thereby, giving us an overall approximation factor of $\sqrt{2}\gamma_z\ell^\frac{1}{z}$. 

Since, $T_f = T_{\phi} = \cO\paren{\ell kd^2zL|S|+\ell d^3L}$, the result follows from \Cref{thm:framework}.
\end{proof}

\section{A Practical Method for Socially Fair Clustering}
\label{sec:experiments}
In this section, we first describe a Lloyd-like heuristic implementation of our framework (see \Cref{alg:algo2}) that is more practical in terms of running time for performing experiments on real-world datasets. Utilizing \Cref{alg:algo2}, we perform an experimental evaluation on multiple benchmark datasets for center-based clustering ($k$-means and $k$-medians objectives) as well as linear subspace clustering\footnote{Implementation can be accessed here: \url{https://github.com/kishen19/Socially-Fair-Clustering}.}

\subsection{Socially Fair \texorpdfstring{\boldmath$(k,z)$}{(k,z)} Lloyd's Heuristic}
In Algorithm \ref{alg:algo2}, rather than iterating over all possible $k$-clusters of the representative set, we perform a Lloyd-like iterative update. This involves three key steps, 1) calculating the clustering (or partition) of the data based on the current set of centers, 2) computing a representative set of the data, and 3) finding the optimal set of $k$ centers on this representative set using \oracle.
Note that unlike described in \Cref{alg}, we construct a representative set at every iteration. This enables us to take advantage of the partitions created at every iteration by constructing a representative set for the $1$-means (and $1$-median) objective for every $X_{ij}$ and taking a union of these representative sets. In most of our experiments, we use a simple uniform random sampling followed by normalization to construct the representative sets, as described in \Cref{alg:algo2}.
We observe in our experimental results that the algorithm finds a good enough solution in very small number of iterations, at which point we stop the algorithm and return the current centers (subspaces).
We note that \cite{GSV2021} also propose a variant of Lloyd's heuristic called Fair-Lloyd where, starting with a uniform random set of clusters of the whole data, in each iteration it finds a set of centers that minimize the socially-fair $k$-means clustering cost for the current set of clusters, and re-assigns all the points to their nearest center.
For arbitrary number of groups, \cite{GSV2021} proposed to find the centers via a heuristic based on the multiplicative weights update method.
Different from this, we work with a representative sample of the data, and in each iteration, we find $\alpha$-approximate centers for the clusters in that iteration, where $\alpha$ is as given in \Cref{sec:cvxprog}.
We evaluate Algorithm \ref{alg:algo2}~with $k$-means and $k$-median objectives for center-based clustering, and with $(q,1,2)$ linear subspace clustering objective (also known as \emph{Linear Subspace Approximation}).
\begin{algorithm}
\caption{Socially Fair $(k,z)$ Lloyd's Heuristic}
\label{alg:algo2}
\begin{algorithmic}[1]
\Require The set of points $X$ and their group memberships, the clustering parameters $k$, $z$, and the sample size $M$.
\Ensure Cluster centers $C$.
\State Initialize the clusters $\cP(X,k)=\paren{P_1(X), P_2(X), \ldots, P_k(X)}$ uniformly at random.

\For {$T$ iterations}
\State $X_{ij} := P_i(X) \cap X_j, \forall i \in [k], j \in [\ell]$.
    
\State    Construct $S_{ij} $ by sampling $M$ points from $X_{ij}$ uniformly at random and set the weight of each point to be $|X_{ij}|/M$. (Note: If $M>|X_{ij}|$, then $S_{ij} = X_{ij}$ with weights as $1$).
    
\State    $S := \bigcup_{i \in [k]} \bigcup_{j \in [\ell]} S_{ij}$.\label{step:coreset}
    
 \State   $\widehat{\cP} := \paren{P_1(X) \cap S, P_2(X) \cap S, \ldots, P_k(X) \cap S}$.
    
 \State   $C =  \text{\oracle}\paren{\widehat{\cP}}$.
    
\State    Update the clusters  $\cP$ by assigning the points in $X$ to their nearest center in $C$.
\EndFor
\State \Return cluster centers $C$.
\end{algorithmic}
\end{algorithm}

\textbf{Datasets.} We experiment over a diverse set of real-world datasets which comprise of various sensitive features: 1) \textit{Credit}, 2) \textit{Adult Income}, 3) \textit{Bank}, 4) \textit{German Credit}, and 5) \textit{Skillcraft}.
All the datasets are obtained from the UCI repository \cite{Dua:2019}.
Table~\ref{tab:datasets} summarizes the size of the dataset, the total number of features in the dataset and the sensitive features based on which the items in the dataset are partitioned into disjoint groups.

\begin{table}[!ht]
    \centering
    \caption{Datasets}
    \label{tab:datasets}
    \setcellgapes{2.5pt}\makegapedcells
    \resizebox{1\textwidth}{!}{
    \begin{tabular}{|l|c|c|c|c|}
        \hline
        \textbf{Dataset} & \textbf{\#samples} & \textbf{\#attr.} & \makecell{\textbf{Sensitive}\\ \textbf{feature}} & \textbf{Groups}\\
        \hline
        Credit~\cite{YEH20092473} & $30000$ & $23$ & Education & Higher, Lower \\
        \hline
        \multirow{2}{*}{\makecell[l]{Adult Income~\cite{Dua:2019}}} & \multirow{2}{*}{$48842$} & \multirow{2}{*}{104} & Gender & Male, Female\\
        \cline{4-5}
        &  &  & Race & \makecell[c]{Amer-Indian-Eskim,  Asian-Pac-Islander, Black, White, Other}\\
        \hline
        Bank~\cite{Moro2014ADA} & $41188$ & $63$ & Age & $\le25$, $25$-$60$, $\ge 60$\\
        \hline
        \makecell[l]{German Credit~\cite{Dua:2019}} & $1000$ & $51$ & Age &$\le25$, $25$-$60$, $\ge 60$\\
        \hline
        Skillcraft~\cite{skillcraft} & $3340$ & $20$ & Age & $<21$, $\ge 21$ \\
        \hline
    \end{tabular}}
    
\end{table}

\textbf{Experimental setup.} 
We normalize the continuous attributes to have mean $0$ and variance $1$, and  encode the categorical attributes using one-hot encoding, similar to our baseline \cite{GSV2021}.
It is a common practice to reduce the dimension using PCA as a pre-processing step \cite{DX2004}, also used in the baseline for socially fair $k$-means, Fair-Lloyd \cite{GSV2021}.
In our experiments for socially fair $k$-means and $k$-medians, we use PCA to get the best $k$ dimensional approximation of the data. In all experiments, all the algorithms are initialized with the same set of centers.

\subsection{Socially Fair \texorpdfstring{\boldmath$k$}{k}-means}

\begin{figure*}[!ht]
	\centering
	\begin{subfigure}[b]{0.05\linewidth}
		\centering
		\includegraphics[scale=0.4]{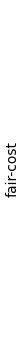} 
	\end{subfigure}
	\begin{subfigure}[b]{0.24\linewidth}
		\centering
		\includegraphics[scale=0.18]{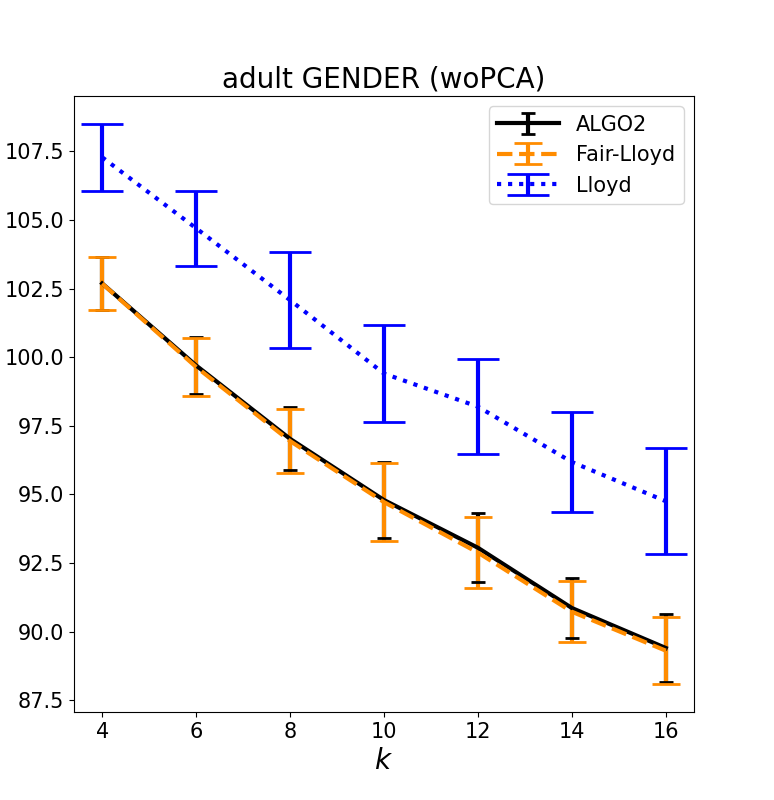} 
	\end{subfigure}
	\begin{subfigure}[b]{0.24\linewidth}
		\centering
		\includegraphics[scale=0.18]{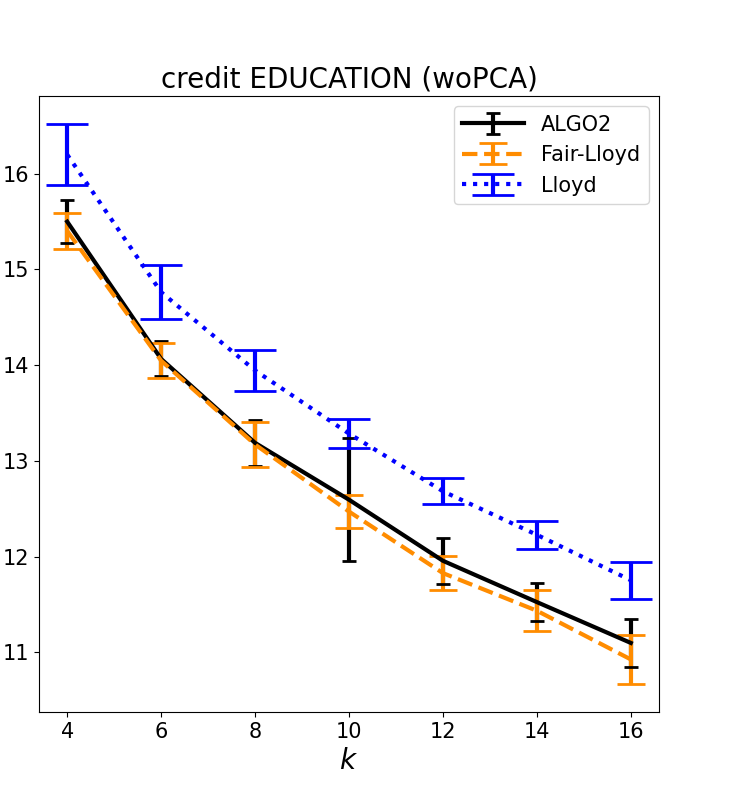} 
	\end{subfigure}
	\begin{subfigure}[b]{0.24\linewidth}
		\centering
		\includegraphics[scale=0.18]{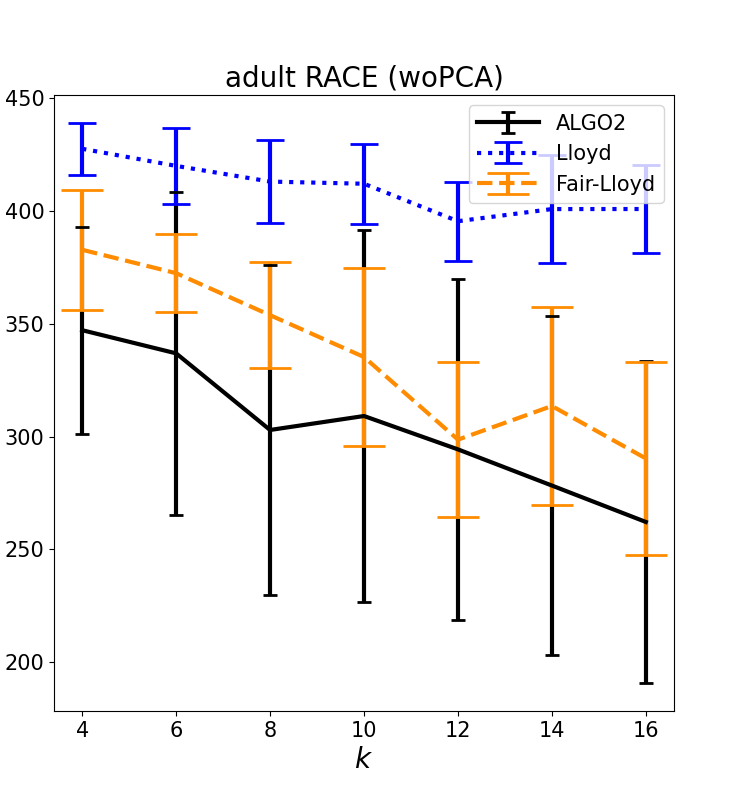} 
	\end{subfigure}
	\begin{subfigure}[b]{0.24\linewidth}
		\centering
		\includegraphics[scale=0.18]{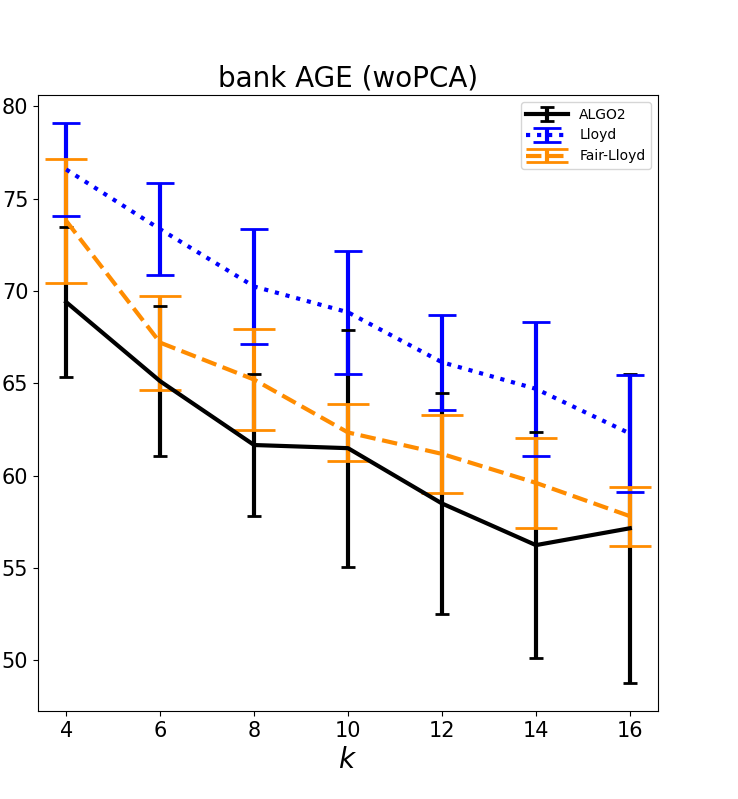} 
	\end{subfigure}
	
	\begin{subfigure}[b]{0.05\linewidth}
		\centering
		\includegraphics[scale=0.4]{results/faircost_cropped.png} 
	\end{subfigure}
    \begin{subfigure}[b]{0.24\linewidth}
		\centering
		\includegraphics[scale=0.18]{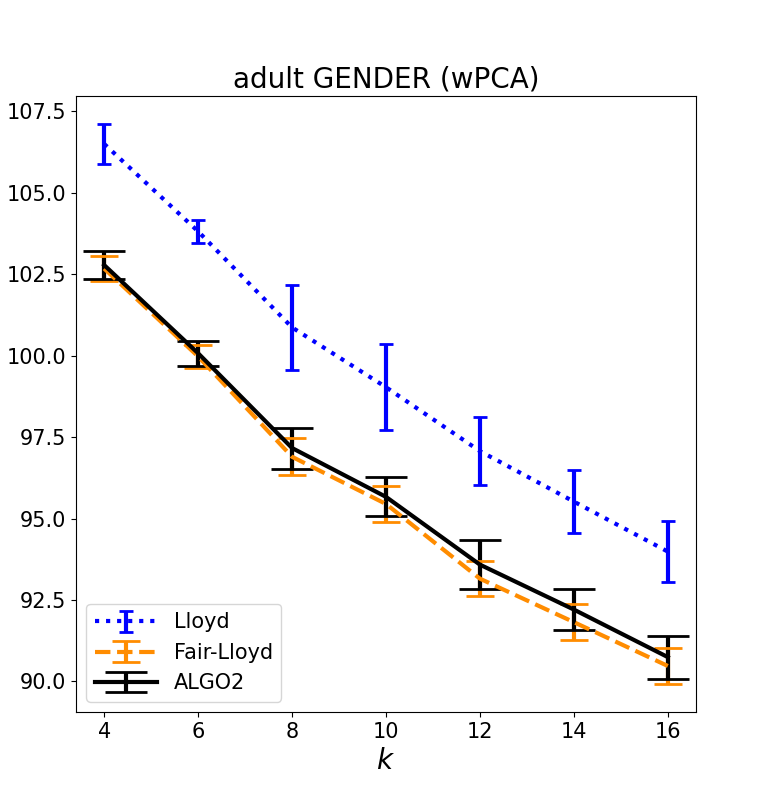}
	\end{subfigure}
	\begin{subfigure}[b]{0.24\linewidth}
		\centering
		\includegraphics[scale=0.18]{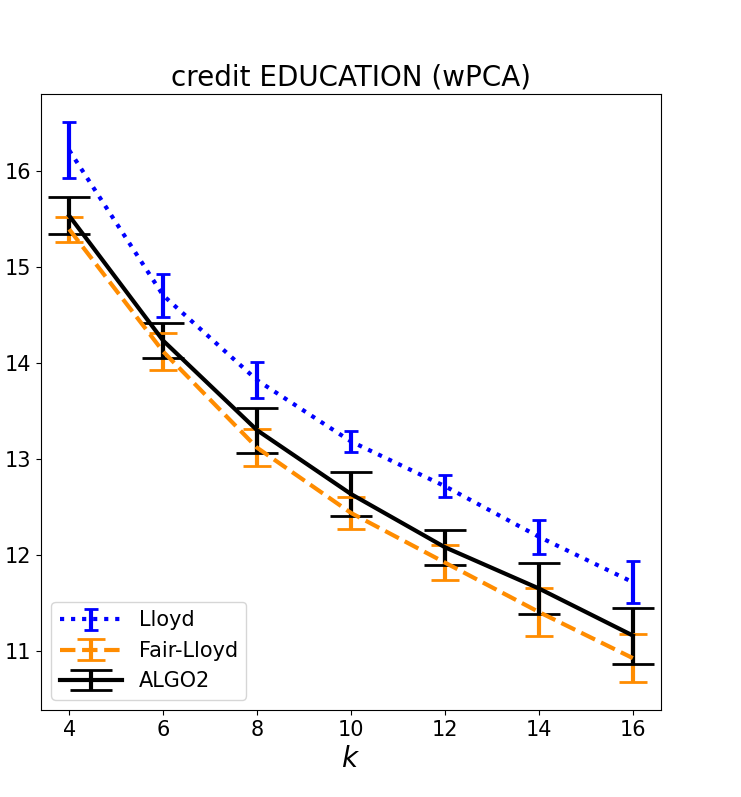} 
	\end{subfigure}
	\begin{subfigure}[b]{0.24\linewidth}
		\centering
		\includegraphics[scale=0.18]{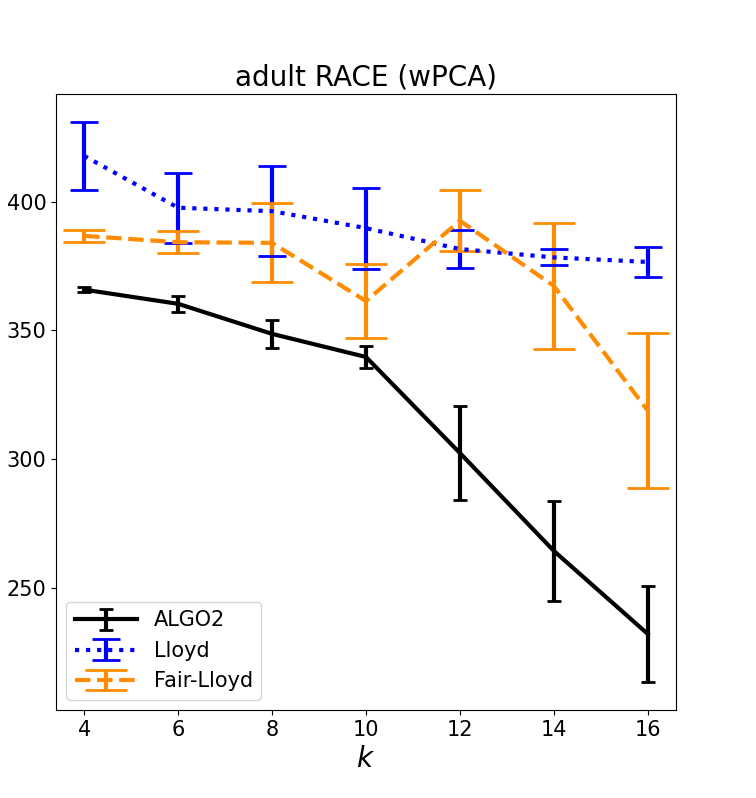} 
	\end{subfigure}
	\begin{subfigure}[b]{0.24\linewidth}
		\centering
		\includegraphics[scale=0.18]{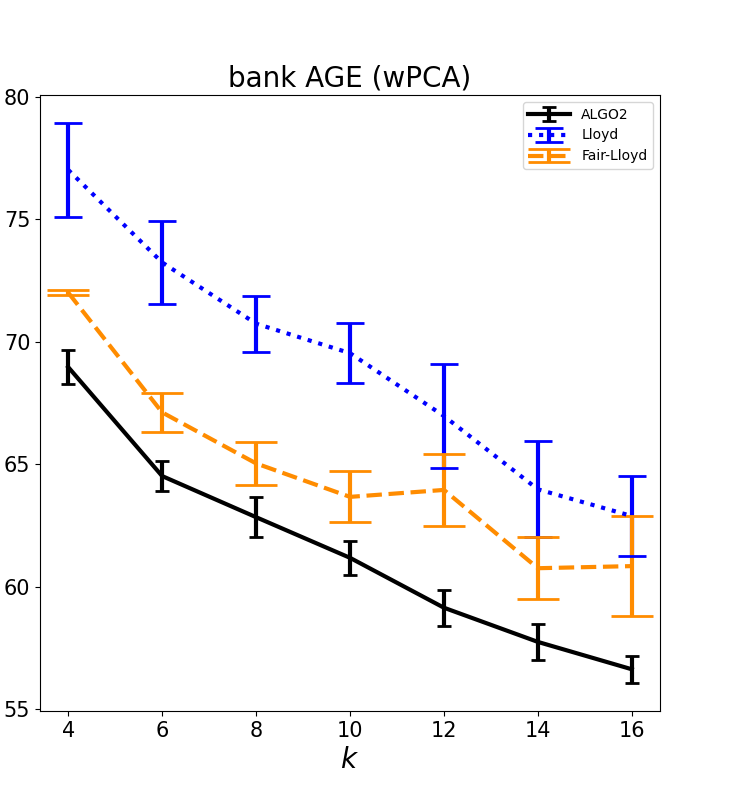} 
	\end{subfigure}

	\caption{All algorithms are run on $50$ different initializations; ALGO2 is run for $20$ iterations, and Fair-Lloyd and Lloyd are run for $100$ iterations. For ALGO2, coreset construction uses $5$ samples of size $M = 100$ for each $P_{ij}$. The plots show mean and standard deviation of the socially fair clustering cost (woPCA = without PCA, wPCA = with PCA, ALGO2 = \Cref{alg:algo2}).}
	\label{fig:kmeans_meanstddev}
\end{figure*}

Before proceeding to the experimental observations, consider the following.
\begin{lemma}[Lemma 2.1 in \cite{KANUNGO200489}]
\label{lem:kmeans}   
Let $X$ be a set of $n$ points in $\mathbb{R}^d$ belonging to $\ell$ groups $X_1,\cdots,X_\ell \subseteq X$. Given a partition $\cP(X,k)=\paren{P_1(X),\cdots,P_k(X)}$, let $X_{ij}$ denote $P_i(X) \cap X_j$. Given a set of centers $C=(c_1,\cdots,c_k)$, $\forall i \in [k], j\in [\ell]$,
    $$\sum_{x\in X_{ij}} \|x-c_i\|^2=\sum_{x \in X_{ij}}\|x-\mu_{ij}\|^2 + |X_{ij}|\cdot\|\mu_{ij}-c_i\|^2,$$
where $\mu_{ij}$ is the centroid of the set $X_{ij}$.
\end{lemma}
Due to \Cref{lem:kmeans} we get a very small ``exact'' representation (or, a $0$-coreset) for $k$-means clustering to use in Algorithm \ref{alg:algo2}, as follows.
We add to the coreset $S$, the mean $\mu_{ij}$ of the points in $X_{ij}$ with weight $|X_{ij}|$, for all $i \in [k]$ and $j \in [\ell]$. 
This gives us a coreset of size $k\ell$. 
However, this would now be a coreset for the modified objective function (with an additive constant) as given in \Cref{lem:kmeans}. Thus, we get,
$$ \cost(C,X_j) = \left(\frac{1}{|X_j|}\left(\sum_{i \in [k]}|X_{ij}|\cdot\|\mu_{ij}-c_i\|^2 + a_j \right)  \right)^{1/2},$$
where, $a_j = \sum\limits_{i \in k}\sum\limits_{x \in X_{ij}} \|x-\mu_{ij}\|^2$. Note that $a_j$ is constant, $\forall j\in [\ell]$.

For the experimental analysis of socially fair $k$-means clustering, we run Algorithm \ref{alg:algo2} (with the coreset defined above) along with the baselines on the Adult Income, Credit and Bank datasets. We perform experiments both with and without the \emph{PCA} pre-processing. We run the algorithms for $50$ different initializations. Algorithm~\ref{alg:algo2} is run for $20$ iterations while the baselines are run for $100$ iterations each.

\textbf{Baselines.} We compare our results with $(i)$ Fair-Lloyd \cite{GSV2021} for socially fair $k$ means clustering.
This is a Lloyd-like heuristic algorithm that performs well for practical purposes. $(ii)$ In our results we also show the socially fair clustering cost obtained by the unconstrained Lloyd's algorithm. 
To the best of our knowledge, there are no implementations of the other socially fair clustering algorithms available for public use.

\textbf{Observations.} As established in \cite{GSV2021}, different groups incur very different costs with Lloyd's algorithm (see \Cref{fig:kmeans_meanstddev}).
On the Adult Income (Gender) and the Credit (Education) datasets, both Fair-Lloyd and Algorithm~\ref{alg:algo2} incur almost equal socially fair $k$ means cost.
In the multi-group case, such as Adult Income (Race) and Bank (Age), Algorithm~\ref{alg:algo2} outperforms Fair-Lloyd.
These observations are consistent with or without PCA preprocessing of the data.
In all the experiments in Figure~\ref{fig:kmeans_meanstddev} both Fair-Lloyd and Algorithm~\ref{alg:algo2} incur much less fair-cost compared to Lloyd.

\subsection{Socially Fair \texorpdfstring{\boldmath$k$}{k}-medians}
\begin{figure*}[!ht]
	\centering
	\begin{subfigure}[b]{0.05\linewidth}
		\centering
		\includegraphics[scale=0.4]{results/faircost_cropped.png} 
	\end{subfigure}
	\begin{subfigure}[b]{0.24\linewidth}
		\centering
		\includegraphics[scale=0.125]{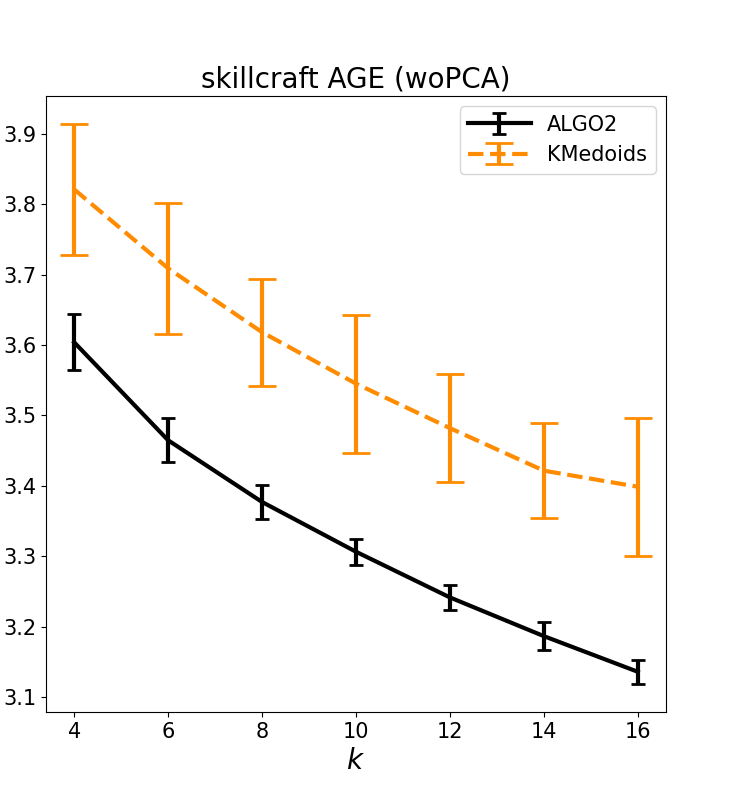}
	\end{subfigure}
	\begin{subfigure}[b]{0.24\linewidth}
		\centering
		\includegraphics[scale=0.17]{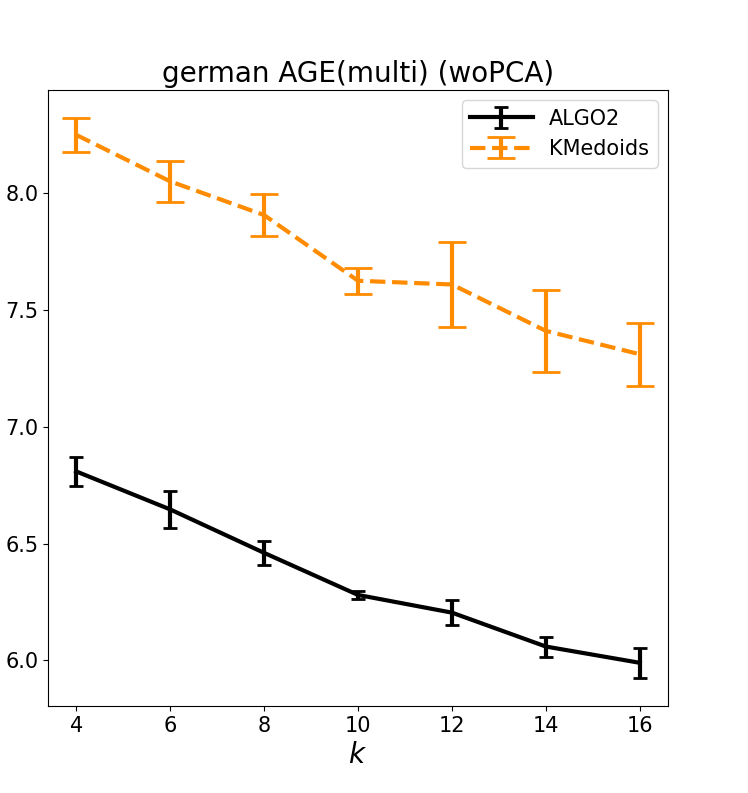}
	\end{subfigure}
	\begin{subfigure}[b]{0.24\linewidth}
		\centering
		\includegraphics[scale=0.125]{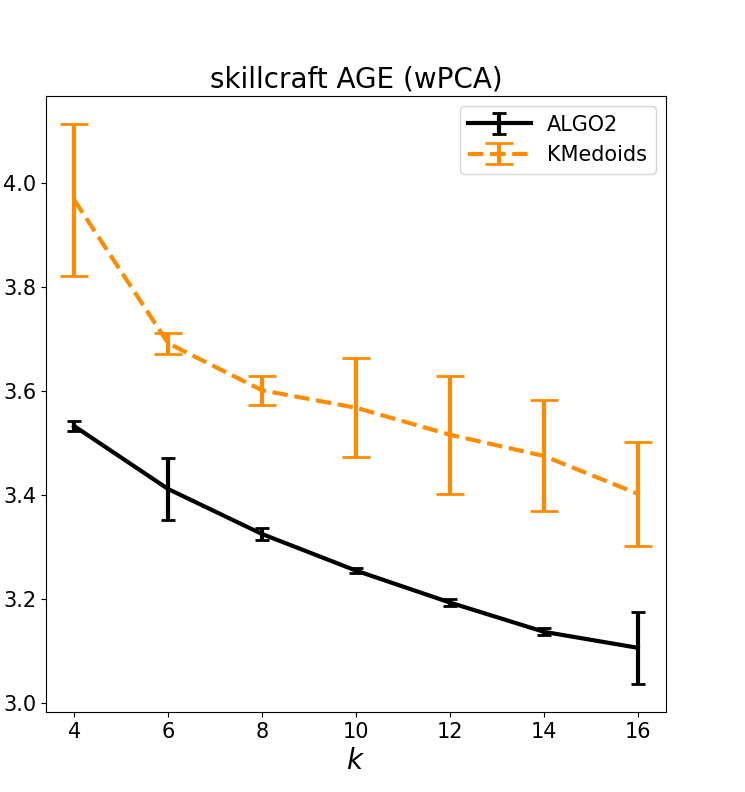}
	\end{subfigure}
	\begin{subfigure}[b]{0.24\linewidth}
		\centering
		\includegraphics[scale=0.125]{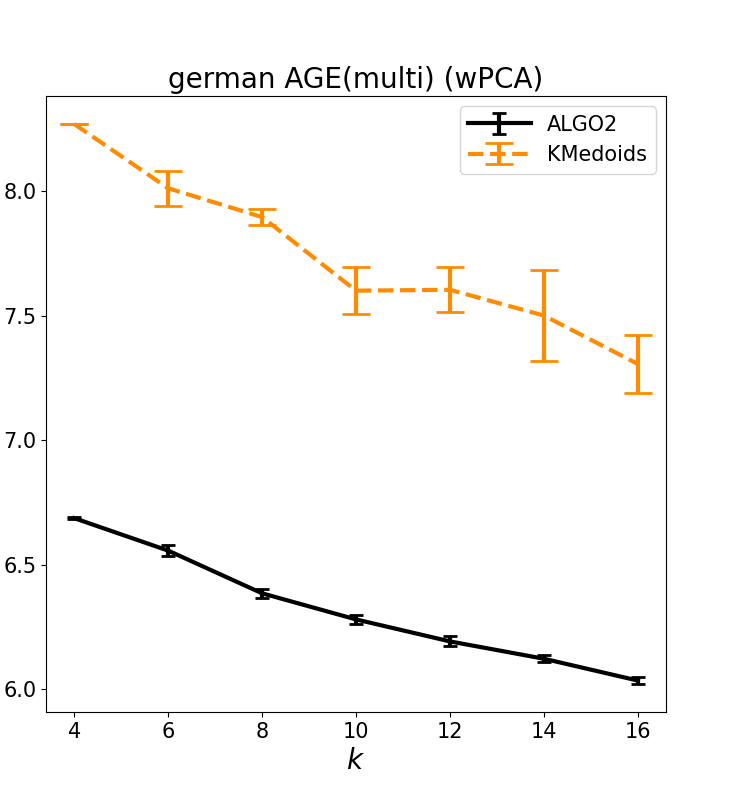}
	\end{subfigure}
	\caption{Both ALGO2 and KMedoids are run on $10$ different initializations for $20$ iterations. For ALGO2, coreset construction uses $5$ samples of size $M = 20$ for each $P_{ij}$. The plots show mean and standard deviation of the socially fair clustering cost (woPCA = without PCA, wPCA = with PCA, ALGO2 = \Cref{alg:algo2}).}
	\label{fig:kmedian_meanstddev}
\end{figure*}

For the socially fair $k$-medians experiments, we run Algorithm \ref{alg:algo2} along with our baseline on the German Credit (Age) and Skillcraft (Age) datasets (\Cref{fig:kmedian_meanstddev}). Both algorithms are initialized with the same set of centers and run on $10$ different initializations for $20$ iterations. We consider both cases of with and without PCA pre-processing.

\textbf{Baseline.} We compare our results with $(i)$ \emph{FasterPAM}~\cite{SCHUBERT2021101804}, a fast and practical $k$-Medoids based algorithm. $(ii)$ We consider the socially fair $k$-medians cost obtained by the unconstrained \emph{FasterPAM} algorithm. Although \emph{FasterPAM} is time-efficient, it is space-inefficient. Hence, we run experiments on datasets with relatively smaller number of data points.

\textbf{Observations.} The baseline $k$-medoids based algorithm incurs very different costs for different groups. Further, the fair-cost obtained by the $k$-medoids algorithm is much higher compared to that of \Cref{alg:algo2} (see \Cref{fig:kmedian_meanstddev}). \Cref{alg:algo2} achieves almost equal costs for different groups in all experiments.

\subsection{Socially Fair Subspace Approximation}
\begin{figure*}[!ht]
	\centering
	\begin{subfigure}[b]{0.03\linewidth}
		\centering
		\includegraphics[scale=0.45]{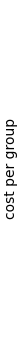} 
	\end{subfigure}
	\begin{subfigure}[b]{0.28\linewidth}
		\centering
		\includegraphics[scale=0.16]{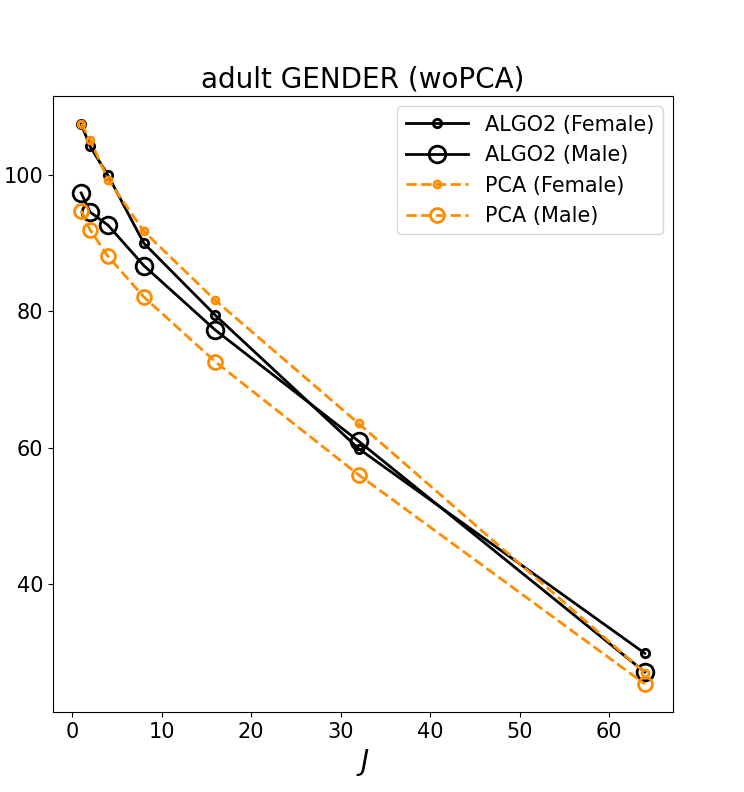}
	\end{subfigure}
	\begin{subfigure}[b]{0.28\linewidth}
		\centering
		\includegraphics[scale=0.16]{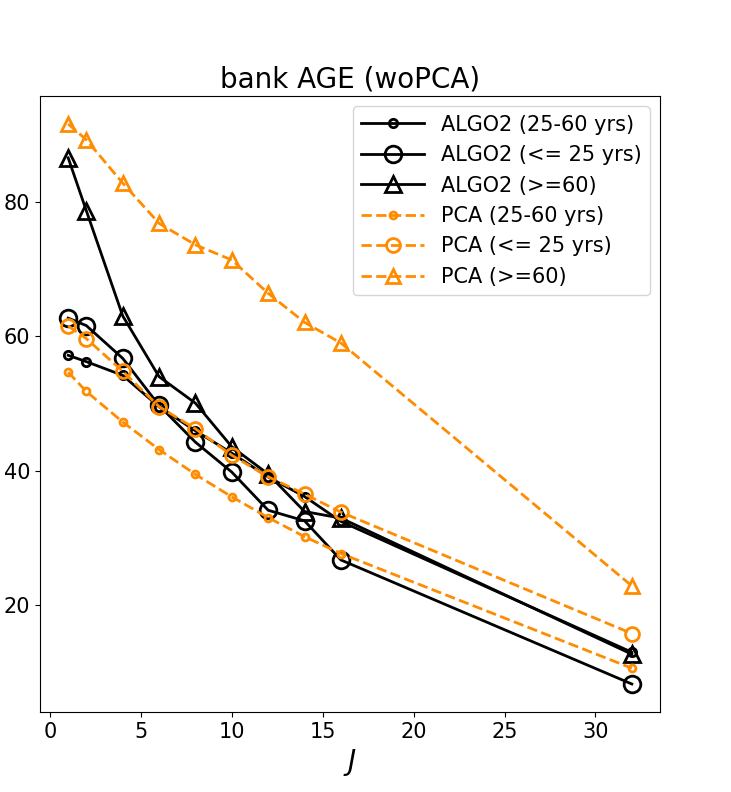} 
	\end{subfigure}
	\begin{subfigure}[b]{0.28\linewidth}
		\centering
		\includegraphics[scale=0.16]{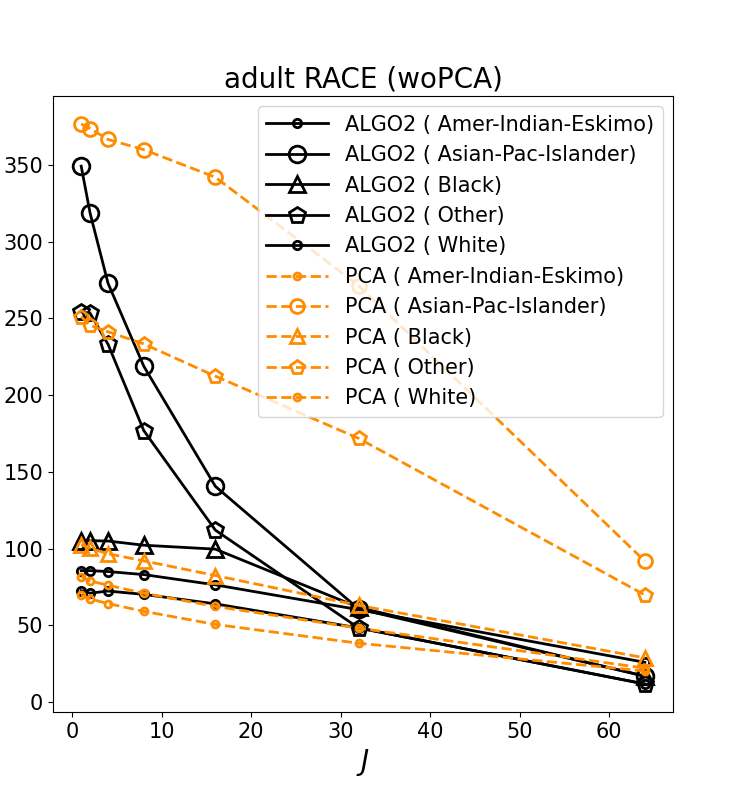} 
	\end{subfigure}
	\caption{For ALGO2, coreset construction uses $20$ samples of size $M = 500$ for each group. Plots show the per group cost incurred (mean and one standard deviation). $J$ on x-axis represents dimension of the linear subspace.}
	\label{fig:subspace_meanstddev}
\end{figure*}

For experiments on the socially fair linear subspace clustering problem, we consider the socially fair variant of the well studied subspace approximation problem, which is equivalent to the socially fair $(q,1,2)$ linear subspace clustering problem, as defined. We run Algorithm \ref{alg:algo2} along with the baseline on the Adult (Gender) dataset for the two groups case.
We also run experiments on the Adult (Race) and Bank (Age) datasets for the multigroups case. Since there is only one subspace (one cluster), the algorithms neither have to be run for multiple iterations nor do we have to consider multiple initializations. 
The $x$-axis in the plots in \Cref{fig:subspace_meanstddev} represents the dimension of the subspace.

\textbf{Baseline.} As a baseline, we run the unconstrained \emph{PCA} algorithm and consider the socially fair linear subspace clustering cost obtained. To the best of our knowledge, there is no implementation of a socially fair linear subspace clustering available.

\textbf{Observations.} Both \cref{alg:algo2} and PCA perform almost equally good on the Adult Income dataset (Gender) in terms of fair-cost (see \Cref{fig:subspace_meanstddev}). However, we observe that for more than two groups (Adult Income (Race) and Bank (Age)), \Cref{alg:algo2} outperforms PCA significantly in terms of fair-cost and achieves almost equal group-wise costs as the dimension of the subspace increases.

\section{Conclusion}
\label{sec:conclusion}
Clustering using center-based objectives and subspaces on large image, text, financial and scientific data sets has a wide range of applications. In order to alleviate harms to different demographic groups arising from inequitable clustering costs across different groups, we study the objective of socially fair clustering. We develop a unified framework to solve socially fair center-based clustering and linear subspace clustering problems, and propose practical and efficient approximation algorithms for them. Our algorithms either closely match or outperform the state-of-the-art baselines on standard real-world data sets in fairness literature. When $p$ and the number of groups $l$ are constants, an interesting open problem is to find the optimal approximation algorithms for socially fair centre-based $\ell_{p}$-norm clustering objectives and affine subspace clustering objectives with running time polynomial in the number of clusters $k$, the number of points $n$, and the dimension $d$.

\paragraph{Acknowledgements.} SG was supported in part by a Google Ph.D. Fellowship award. AL was supported in part by SERB Award ECR/2017/003296, a Pratiksha Trust Young Investigator Award, and an IUSSTF virtual center on “Polynomials as an Algorithmic Paradigm”. AL is also grateful to Microsoft Research for supporting this
collaboration.

\bibliographystyle{alpha}
\bibliography{references}

\appendix

\section{Appendix}
\begin{algorithm}
\caption{Framework for the socially fair linear subspace clustering}
\label{framework_subspace}
\begin{algorithmic}[1]
\Require The set of points $P$ and their group memberships, the numbers $k$, $z$, and $q$.

\Ensure $k$ linear subspaces $V$.

\State Let $S_j$ be an $\varepsilon$-strong coreset for group $j, \forall j\in[\ell]$.
\State Let $S := \bigcup_{j\in[\ell]} S_j$.
\For {each $k$ partitioning of the coreset $\cC = \paren{C_1, C_2, \ldots, C_k}$}
\State $V := \text\oracle(\cC)$.
\State $t := \max_{j \in [\ell]} \cost(V,S_j)$.
\EndFor
\State \Return The linear subspaces $V$ with minimum value of $t$.
\end{algorithmic}
\end{algorithm}
\end{document}